\documentclass{article}
\usepackage{ijcai21}

\usepackage{times}

\usepackage{soul}
\usepackage{url}

\usepackage[utf8]{inputenc}
\usepackage[small]{caption}
\usepackage{booktabs}

\usepackage{amssymb}
\usepackage{graphicx}
\usepackage{amsmath}
\usepackage{xspace}
\usepackage{bm}
\usepackage{amsthm}
\usepackage{comment}
\usepackage{afterpage}
\usepackage{indentfirst}
\usepackage{enumerate}
\usepackage{subfigure}
\usepackage[titlenumbered,ruled,lined,linesnumbered,algo2e]{algorithm2e}
\usepackage[compress]{natbib}
\usepackage{color}
\usepackage{mathtools}
\usepackage[switch]{lineno}
\newtheorem{theorem}{Theorem}
\newtheorem{lemma}[theorem]{Lemma}

\newtheorem{corollary}[theorem]{Corollary}
\theoremstyle{definition}
\newtheorem{definition}{Definition}
\newtheorem{assumption}{Assumption}

\theoremstyle{definition}

\newtheorem{remark}{Remark}

\usepackage[hidelinks]{hyperref}

\newcommand{\nbb}{\mathbb{N}}

\newcommand{\bw}{\mathbf{w}}

\newcommand{\xcal}{\mathcal{X}}
\newcommand{\wcal}{\mathcal{W}}

\newcommand{\zcal}{\mathcal{Z}}

\newcommand{\ycal}{\mathcal{Y}}

\newcommand{\ebb}{\mathbb{E}}

\newcommand{\be}{\mathbf{e}}

\newcommand{\rbb}{\mathbb{R}}

\title{Stability and Generalization for Randomized Coordinate Descent\footnote{To appear in IJCAI 2021}}

\author{
Puyu Wang$^1$\and
Liang Wu$^2$\And
Yunwen Lei$^{3}$
\affiliations
$^1$School of Mathematics, Northwest University, Xi'an 710127, China\\
$^2$Center of Statistical Research, School of Statistics, Southwestern University of Finance and Economics, Chengdu 611130, China\\
$^3$School of Computer Science, University of Birmingham, Birmingham B15 2TT, UK\\
\emails
 wangpuyu@stumail.nwu.edu.cn,
wuliang@swufe.edu.cn,
y.lei@bham.ac.uk
}

\begin{document}

\maketitle

\begin{abstract}
  Randomized coordinate descent (RCD) is a popular optimization algorithm with wide applications in solving various machine learning problems, which motivates a lot of theoretical analysis on its convergence behavior. As a comparison, there is no work studying how the models trained by RCD would generalize to test examples. In this paper, we initialize the generalization analysis of RCD by leveraging the powerful tool of algorithmic stability. We establish argument stability bounds of RCD for both convex and strongly convex objectives, from which we develop optimal generalization bounds by showing how to early-stop the algorithm to tradeoff the estimation and optimization. Our analysis shows that RCD enjoys better stability as compared to stochastic gradient descent.
\end{abstract}


\section{Introduction}

Randomized coordinate descent (RCD) is a popular method for solving large-scale optimization problems which are ubiquitous in the big-data era ~\citep{nesterov2012efficiency,richtarik2014iteration,richtarik2016distributed}. As an iterative algorithm, it iteratively updates a single randomly chosen coordinate along the negative direction of the derivative while keeping the other coordinates fixed. Due to its ease of implementation and high efficiency, RCD has found wide applications in various areas such as compressed sensing, network problems and optimization~\citep{richtarik2014iteration}. In particular, the conceptual and algorithmic simplicity makes it especially useful for large-scale problems for which even the simplest full-dimensional vector operations are very expensive~\citep{nesterov2012efficiency}.

The popularity of RCD motivates a lot of theoretical analysis to understand its empirical behavior. Specifically, iteration complexities of RCD are well studied in the literature under different settings (e.g., convex/strongly convex~\citep{nesterov2012efficiency}, smooth/nonsmooth cases~\citep{richtarik2014iteration,richtarik2016distributed,lu2015complexity}) for different variants (e.g., distributed RCD~\citep{richtarik2016distributed},  accelerated RCD~\citep{nesterov2012efficiency,ren2017distributed,gu2018accelerated,li2020complexity,chen2016accelerated}, RCD for primal-dual problems~\citep{qu2016sdna} and RCD for saddle problems~\citep{zhu2016stochastic}). These discussions concern how the empirical risks of the models trained by RCD would decay along the optimization process. As a comparison, there is little analysis on how these models would behave on testing examples, which is what really matters in machine learning. Actually, if the models are very complicated, it is very likely that the models would admit a small empirical risk or even interpolate the training examples but meanwhile suffer from a large test error. This discrepancy between training and testing, as referred to as overfitting, is a fundamental problem in machine learning~\citep{bousquet2002stability}. The existing convergence rate analysis of RCD is not enough to fully understand why the models trained by RCD have a good prediction performance in real applications. In particular, it is not clear how the optimization and statistical behavior of RCD would change along the optimization process, which is useful for  designing efficient models in practice. For example, generalization analysis provides a principled guideline on how to stop the algorithm appropriately for a best generalization. 

In this paper, we aim to bridge the generalization and optimization of RCD by leveraging the celebrated concept of algorithmic stability. We establish stability bounds of RCD as measured by several concepts, including $\ell_1$-argument stability, $\ell_2$-argument stability and uniform stability. Under standard assumptions on smoothness, Lipschitz continuity and convexity of objective functions, we show clearly how the stability and the optimization error would behave along the learning process. This suggests a principled way to early-stop the algorithm to get a best generalization behavior. We consider convex, strongly convex and nonconvex objective functions. In the convex and strongly convex cases, we develop minimax optimal generalization bounds of the order $O(1/\sqrt{n})$ and $O(1/n)$ respectively, where $n$ is the sample size. Our analysis not only suggests that RCD has a better stability than stochastic gradient descent (SGD), but also is able to exploit a low noise condition to get an optimistic bound $O(1/n)$ in the convex case. Finally, we develop generalization bounds with high probability which are useful to understand the robustness and variation of the training algorithm~\citep{feldman2019high}.



\section{Related Work\label{sec:work}}
\subsection{Randomized Coordinate Descent}

RCD was widely used to solve large-scale optimization problems in machine learning, including linear SVMs~\citep{chang2008coordinate}, $\ell_1$-regularized models for sparse learning~\citep{shalev2009stochastic} and low-rank matrix learning~\citep{hu2019low}.
The convergence rate of RCD and its accelerated variant were studied in the seminal work~\citep{nesterov2012efficiency}, where the advantage of RCD over deterministic algorithms is clearly illustrated. These results were extended to structure optimization problems where the objective function consists of a smooth data-fitting term and a nonsmooth regularizer~\citep{richtarik2014iteration,lu2015complexity}. RCD was also adapted to distributed data analysis\citep{richtarik2016distributed,sun2017asynchronous,xiao2019dscovr}, primal-dual optimization ~\citep{qu2016sdna} and  privacy-preserving problems~\citep{damaskinos2020differentially}. All these discussions consider the convergence rate of optimization errors for RCD. As a comparison, we are interested in the generalization behavior of models trained by RCD, which is the ultimate goal in machine learning.

\subsection{Stability and Generalization}
We now review the related work on algorithmic stability and its application on generalization analysis. The framework of algorithmic stability was established in a seminal paper~\citep{bousquet2002stability}, where the important uniform stability was introduced. This algorithmic stability was extended to study randomized algorithms in \citet{elisseeff2005stability}. Other than uniform stability, several other stability measures including hypothesis stability~\citep{bousquet2002stability}, on-average stability~\citep{shalev2010learnability} and argument stability~\citep{liu2017algorithmic} have been introduced in statistical learning theory, whose connection to learnability has been established~\citep{mukherjee2006learning,shalev2010learnability}. The uniform stability of stochastic gradient descent (SGD) was established for learning with (strongly) convex, smooth and Lipschitz loss functions~\citep{hardt2016train}. This motivates the recent work of studying  generalization of stochastic optimization algorithms via several stability~\citep{meng2017generalization,charles2018stability,kuzborskij2018data,yin2018gradient,yuan2019stagewise,lei2020fine,bassily2020stability,lei2020sharper,lei2021sharper,wang2021differentially,yang2021stability}. For example, an on-average model stability~\citep{lei2020fine} has been proposed to remove the smoothness assumption or Lipschitz continuity assumption in \citet{hardt2016train}. Recently, elegant concentration inequalities have been developed to get high-probability bounds via uniform stability~\citep{feldman2019high,bousquet2019sharper}. To our best knowledge, the algorithmic stability of RCD has not been studied yet, which is the topic of this paper.

\section{Problem Formulation\label{sec:formulation}}

Let $\rho$ be a probability measure defined over a sample space $\zcal=\xcal\times\ycal$, where $\xcal$ is an input space and $\ycal\subset\rbb$ is an output space. We aim to build a parametric model $h_{\bw}:\xcal\mapsto\rbb$, where $\bw$ is the model parameter which belongs to the parameter space $ \wcal\subseteq\rbb^d$. The performance of the model $h_{\bw}$ on a single example $z$ can be measured by a nonnegative loss function $f(\bw;z)$. The quality of a model can be quantified by a population risk $F(\bw)=\ebb_z[f(\bw;z)]$, where $\ebb_z$ denotes the expectation w.r.t. $z$. We wish to approximate the best model $\bw^*\in\arg\min_{\bw\in\wcal}F(\bw)$. However, the probability measure is often unknown and we only have access to a training sample $S=\{z_1,z_2,\ldots,z_n\}$ drawn independently from $\rho$. The empirical behavior of
$h_{\bw}$ on $S$ can be measured by a empirical risk $F_S(\bw)=\frac{1}{n}\sum_{i=1}^{n}f(\bw;z_i)$.

\noindent\textbf{Notations.}
For any $\mathbf{x}\in \rbb^d$, we denote the norm $\|\mathbf{x}\|_p=\big(\sum_{i=1}^{d}|\mathbf{x}_i|^p\big)^{1/p}$ for $p\geq1$.
For any $m\in\nbb$, we use the notation $[m]:=\{1,\ldots,m\}$. And we use the notation $B=O(\tilde{B})$ if there exists a constant $c_0>0$ such that $B\le c_0 \tilde{B}$, and use $B\asymp \widetilde{B}$ if there exist constants $c_1,c_2>0$ such that $c_1\widetilde{B}<B\leq c_2\widetilde{B}$. We say $g:\wcal\mapsto\rbb$ is $L$-smooth if $\|\nabla g(\bw)-\nabla g(\bw')\|_2\leq L\|\bw-\bw'\|_2$ for all $\bw,\bw'\in\wcal$.



We apply a randomized algorithm $A$ to the sample $S$ and get an output model $A(S)\in\wcal$.
We are interested in studying the {\em excess generalization error} $F(A(S))-F(\bw^*)$. Since $\ebb[F_S(\bw^*)]=F(\bw^*)$, we can decompose the excess generalization error by
\begin{align}
\ebb_{S,A}\big[F(A(S))\!-&\!F(\bw^*)\big]\!=\!\ebb_{S,A}\big[F(A(S))\!-\!F_S(A(S))\big]\notag\\
&+\ebb_{S,A}\big[F_S(A(S))-F_S(\bw^*)\big].\label{decomposition}
\end{align}
We refer to the first term $\ebb_{S,A}\big[F(A(S))-F_S(A(S))\big]$ as the estimation error, and the second term $\ebb_{S,A}\big[F_S(A(S))-F_S(\bw^*)\big]$ as the optimization error. A standard approach to control estimation error is to study the algorithmic stability of the algorithm $A$, i.e., how the model would change if we change the training sample by a single example. There are several variants of stability measures including  the uniform stability, hypothesis stability, on-average stability and argument stability~\citep{bousquet2002stability,hardt2016train,elisseeff2005stability}, among which the uniform stability is the most popular one.
\begin{definition}[Uniform Stability\label{def:unif-stab}]
  A randomized algorithm $A$ is $\epsilon$-uniformly stable if for all training datasets $S,\widetilde{S}\in\zcal^n$ that differ by at most one example, we have
  \begin{equation*}
  \sup_z\big[f(A(S);z)-f(A(\widetilde{S});z)\big]\leq\epsilon.
  \end{equation*}
\end{definition}
In this paper we consider the on-average argument/model stability~\citep{lei2020fine}, an advantage of which is that it can imply better generalization bounds without a Lipschitz continuity assumption on loss functions.
\begin{definition}[On-average Argument Stability\label{def:aver-stab}]
  Let $S=\{z_1,\ldots,z_n\}$ and $S'=\{z'_1,\ldots,z'_n\}$ be drawn independently from $\rho$. For any $i=1,\ldots,n$, define
  $S^{(i)}=\{z_1,\ldots,z_{i-1},z_i',z_{i+1},\ldots,z_n\}$ as the set formed from $S$ by replacing $z_i$ with $z_i'$.
  We say a randomized algorithm $A$ is $\ell_1$ on-average argument $\epsilon$-stable if
  $
  \ebb_{S,S',A}\big[\frac{1}{n}\sum_{i=1}^{n}\|A(S)-A(S^{(i)})\|_2\big]\leq\epsilon,
  $
  and $\ell_2$ on-average argument $\epsilon$-stable if
  $
  \ebb_{S,S',A}\big[\frac{1}{n}\sum_{i=1}^{n}\|A(S)-A(S^{(i)})\|_2^2\big]\leq\epsilon^2.
  $
\end{definition}

Lemma \ref{thm:gen-model-stab} \citep{lei2020fine} gives a connection between on-average argument stability and estimation errors. Assumption \ref{ass:lipschitz} holds for popular loss functions including logistic loss and Huber loss.
\begin{assumption}\label{ass:lipschitz}
  Let $G_1,G_2>0$. Assume for all $\bw\in\wcal$ and $z\in\zcal$,
  $ \|\nabla f(\bw;z)\|_1\leq G_1 $ and  $\|\nabla f(\bw;z)\|_2\leq G_2.$
\end{assumption}
\begin{lemma}[Generalization via Argument Stability\label{thm:gen-model-stab}]
Let $S,S'$ and $S^{(i)}$ be constructed as Definition \ref{def:aver-stab}.
\begin{enumerate}[(a)]
  \item If Assumption \ref{ass:lipschitz} holds, then
  \begin{multline*}
  \big|\ebb_{S,A}\big[F_S(A(S))-F(A(S))\big]\big|\leq \\
  \frac{G_2}{n}\ebb_{S,S',A}\Big[\sum_{i=1}^{n}\|A(S)-A(S^{(i)})\|_2\Big].
  \end{multline*}
  \item If for any $z$, the function $\bw\mapsto f(\bw;z)$ is nonnegative and $L$-smooth, then for any $\gamma>0$ we have
  \begin{multline*}
    \ebb_{S,A}\big[F(A(S))-F_S(A(S))\big]  \leq \frac{1}{\gamma}\ebb_{S,A}\big[F_S(A(S))\big]\\
    +\frac{L(1+\gamma)}{2n}\sum_{i=1}^{n}\ebb_{S,S',A}\big[\|A(S^{(i)})-A(S)\|_2^2\big].
  \end{multline*}
\end{enumerate}
\end{lemma}
In this paper, we consider the specific RCD method widely used in large-scale learning problems. Let $\bw_1\in\wcal$ be the initial point.
At the $t$-th iteration it first  randomly selects a single coordinate $i_t\in[d]$, and then performs the update along the $i_t$-th coordinate as~\citep{nesterov2012efficiency}
\begin{equation}\label{RCD}
\bw_{t+1}=\bw_t-\eta_t\nabla_{i_t}F_S(\bw_t)\be_{i_t},
\end{equation}
where $\nabla_{i}g$ denotes the derivative of $g$ w.r.t. the $i$-th coordinate and $\be_i$ is a vector in $\rbb^d$ with the $i$-th coordinate being $1$ and other coordinates being $0$. Here {$\{\eta_t\}$} is a nonnegative stepsize sequence. It is clear that RCD sequentially updates a randomly selected coordinate while keeping others fixed. In this paper,  we consider the update of only a single coordinate per iteration. Our discussions can be readily extended to randomized block coordinate descent where the coordinates are partitioned into blocks, and each block of coordinates is updated per iteration~\citep{nesterov2012efficiency}.

\section{Stability of RCD\label{sec:stability}}
In this section, we present our stability bounds of RCD.
To this aim,  we first introduce several standard assumptions~\citep{nesterov2012efficiency}. The first assumption is the convexity of the empirical risk. Note we do not require the convexity of each loss function, which is used in the stability analysis of SGD~\citep{hardt2016train,kuzborskij2018data}. 
\begin{assumption}\label{ass:convex}
  For any training dataset set $S$, $F_S$ is convex.
\end{assumption}

Our second assumption is the coordinate-wise smoothness of the empirical risk.
\begin{definition}\label{def:coo-smooth}
  We say a differentiable function $g:\wcal\mapsto\rbb$ has coordinate-wise Lipschitz continuous gradients with parameter $\widetilde{L}>0$ if the following inequality holds for all $\alpha\in\rbb,\bw\in\wcal,i\in[d]$
  \[
  g(\bw+\alpha\be_i)\leq g(\bw)+\alpha\nabla_ig(\bw)+\widetilde{L}\alpha^2/2.
  \]
\end{definition}
\begin{assumption}\label{ass:smooth}
  For any training dataset $S$, $F_S$ is $L$-smooth and has coordinate-wise Lipschitz continuous gradients with parameter $\widetilde{L}>0$.
\end{assumption}


\noindent\textbf{Convex case}. Under these assumptions, we establish stability bounds of RCD. Part (a) of Theorem \ref{thm:stab-bound-rcd} considers the $\ell_1$ on-average argument stability, while Part (b) considers the $\ell_2$ on-average argument stability. The proof is given in Section~\ref{sec:proof}.
\begin{theorem}\label{thm:stab-bound-rcd}
Let Assumptions \ref{ass:convex}, \ref{ass:smooth} hold.
Let $\{\bw_t\}, \{\bw_t^{(i)}\}$ be given by \eqref{RCD} with $\eta_t\leq2/\widetilde{L}$ based on $S$ and $S^{(i)}$, respectively.
\begin{enumerate}[(a)]
  \item If Assumption \ref{ass:lipschitz} holds, then 
  \begin{equation}\label{stab-bound-rcd-l1}
  \frac{1}{n}\sum_{i=1}^{n}\ebb_{S,S',A}\Big[\|\bw_{t+1}-\bw_{t+1}^{(i)}\|_2\Big] \leq \frac{2G_1}{nd}\sum_{k=1}^{t}\eta_k.
  \end{equation}
  \item For any $p>0$ the $\ell_2$ on-average argument stability can be bounded by
    \begin{multline}\label{stab-bound-rcd}
    \frac{1}{n}\sum_{i=1}^{n}\ebb_A\big[\|\bw_{t+1}-\bw_{t+1}^{(i)}\|_2^2\big]
    \leq \frac{4L(1+1/p)}{n^2d}\times\\
    \sum_{j=1}^{t}\big(1+p\big)^{t-j}\eta_j^2\ebb_A\big[F_S(\bw_j)+F_{S'}(\bw_j)\big].
\end{multline}
\end{enumerate}

\end{theorem}
\begin{remark}
  We now compare the above results with the related work. Under Assumptions \ref{ass:lipschitz}, \ref{ass:convex} and \ref{ass:smooth}, it was shown that SGD with $t$ iterations enjoys the $\ell_1$ argument stability bound $O(\frac{G_2}{n}\sum_{k=1}^{t}\eta_k)$. Eq. \eqref{stab-bound-rcd-l1} shows that RCD admits a better stability since there is a $d$ in the denominator. Note that in the worst case we can choose $G_1\leq\sqrt{d}G_2$ for which our stability bounds of RCD are of the order $O(\frac{G_2}{n\sqrt{d}}\sum_{k=1}^{t}\eta_k)$.
  A notable property of Part (b) is that the stability bound \eqref{stab-bound-rcd} does not require the Lipschitz condition as $\|\nabla f(\bw;z)\|_2\leq G_2$, which is widely used in the existing stability analysis~\citep{hardt2016train,charles2018stability,kuzborskij2018data}. Indeed, a key point here is that we replace the Lipschitz constant $G_2$ by empirical/population risks $F$ and $F_S$. Since we are minimizing the empirical risk by RCD, it is reasonable that $F$ and $F_S$ would be small and in this case the algorithm would be more stable. This gives an intuitive connection between stability and optimization: a small optimization error is also beneficial to improve stability.
\end{remark}

\noindent\textbf{Strongly convex case}. Theorem \ref{thm:stab-bound-rcd} shows the stability becomes worse as we run more iterations. In the following theorem, we show the stability can be further improved if we impose a strong convexity assumption.
\begin{assumption}\label{ass:sc}
  Assume for all $S,i\in[d],\bw\in\wcal$, the function $v\mapsto F_S(\bw+v\be_i)$ is $\sigma$-strongly convex, i.e.,
  \begin{multline*}
  F_S(\bw+v\be_i)\geq F_S(\bw+v'\be_i)+
  (v-v')\nabla_iF_S(\bw+v'\be_i)\\+\sigma(v-v')^2/2,\quad\forall v,v'\in\rbb.
  \end{multline*}
\end{assumption}
\begin{theorem}\label{thm:stab-bound-rcd-sc}
Let Assumptions \ref{ass:lipschitz}, \ref{ass:smooth}, \ref{ass:sc} hold.
Let $\{\bw_t\}, \{\bw_t^{(i)}\}$ be produced by \eqref{RCD} with $\eta_t\leq1/\widetilde{L}$ based on $S$ and $S^{(i)}$, respectively. Then we have
\begin{equation}\label{stab-bound-rcd-sc}
\ebb_A[ \|\bw_{t+1}-\bw_{t+1}^{(i)}\|_2]\leq \frac{4G_1}{n\sigma}.
\end{equation}
\end{theorem}
\begin{remark}
  Stability bounds of the order $O(1/(n\sigma))$ were established for SGD under a strong convexity assumption~\citep{hardt2016train}, which are extended to RCD here. Another difference is that the stability bounds in \citet{hardt2016train} are established for either the constant stepsize sequence {$\eta_t\equiv\eta$} or the specific stepsize sequence $\eta_t=1/(t\sigma)$. As a comparison, our results apply to general stepsizes.
\end{remark}

\noindent\textbf{Nonconvex case}. We now present stability bounds for nonconvex problems, which are ubiquitous in the modern machine learning.
We denote $\prod_{k=t+1}^{t}\big(1+\widetilde{L}\eta_kd^{-\frac{1}{2}}\big)=1$.
\begin{theorem}\label{thm:stab-nonconvex}
Let Assumptions \ref{ass:lipschitz} and \ref{ass:smooth} hold.
Let $\{\bw_t\}, \{\bw_t^{(i)}\}$ be produced by \eqref{RCD} based on $S$ and $S^{(i)}$, respectively. Then
\[
\ebb_A\big[\|\bw_{t+1}-\bw_{t+1}^{(i)}\|_2\big]\leq \frac{2G_1}{nd}\sum_{j=1}^{t}\eta_j\prod_{k=j+1}^{t}\big(1+\widetilde{L}\eta_kd^{-\frac{1}{2}}\big).
\]
\end{theorem}

\noindent\textbf{Almost sure bounds}. The above theorems consider stability bounds in expectation. The following theorem 
gives almost sure stability bounds, which is useful to develop high-probability generalization bounds. We need a coordinate-wise Lipschitz continuity assumption.
\begin{assumption}\label{ass:lipschitz-inf}
  For all $S$ and $[i]\in[d]$, assume $|\nabla_i F_S(\bw)|\leq \widetilde{G}$ for all $\bw\in\wcal$.
\end{assumption}
\begin{theorem}\label{thm:stab-bound-rcd-hp}
Let Assumptions \ref{ass:lipschitz}, \ref{ass:convex}, \ref{ass:smooth}, \ref{ass:lipschitz-inf} hold. Then RCD with $T$ iterations is $\frac{2G_2\widetilde{G}}{n}\sum_{t=1}^{T}\eta_t$-uniformly stable.
\end{theorem}

\section{Generalization of RCD\label{sec:generalization}}
In this section, we use the stability bounds in the previous section to develop generalization bounds for RCD. According to \eqref{decomposition}, we need to tackle the optimization errors for a complete generalization analysis.
The following lemma is a slight variant of the optimization error bounds in \citet{nesterov2012efficiency}. 
\begin{lemma}[Optimization Errors]\label{lem:opt-rcd}
Let Assumptions \ref{ass:convex}, \ref{ass:smooth} hold.
Let $\{\bw_t\}$ be produced by \eqref{RCD} with nonincreasing step sizes satisfying $\eta_t\leq2/\widetilde{L}$. Then
\begin{equation}\label{rcd-a}
\ebb_A[F_S(\bw_t)-F_S(\bw)]
\leq \frac{d}{2\sum_{j=1}^{t}\eta_j}\Big(\|\bw_1-\bw\|_2^2  + 2\eta_1F_S(\bw_1)\Big).
\end{equation}
If $F_S$ is $\sigma$-strongly convex, then ({$\bw_S=\arg\min_{\bw\in\wcal}F_S(\bw)$})
\begin{equation}\label{rcd-c}
\ebb_A\big[F_S(\bw_{t+1})-F_S(\bw_S)\big]\!\leq\!\big(1\!-\!\eta_t\sigma/d\big)\ebb_A\big[F_S(\bw_{t})-F_S(\bw_S)\big].
\end{equation}
\end{lemma}


\noindent\textbf{Convex Case}. We first use the technique of $\ell_1$ on-average argument stability to develop generalization bounds under a Lipschitz continuity assumption, i.e., Assumption \ref{ass:lipschitz}.
\begin{theorem}\label{thm:gen-rcd-l1}
Let Assumptions \ref{ass:lipschitz}, \ref{ass:convex}, \ref{ass:smooth} hold.
Let $\{\bw_t\}$ be produced by \eqref{RCD} with {$\eta_t\equiv\eta\leq2/\widetilde{L}$}. Then
\begin{multline}\label{gen-rcd-l1-b}
  \ebb_{{S,A}}\big[F(\bw_T)-F(\bw^*)\big]\leq \frac{2G_1G_2T\eta}{nd}+\\
  \frac{d\|\bw_1-\bw^*\|_2^2}{2T\eta} +\frac{dF(\bw_1)}{T}.
\end{multline}
If we choose $T\asymp d\sqrt{n}$, then
\begin{equation}\label{gen-rcd-l1}
\ebb_{{S,A}}\big[F(\bw_T)-F(\bw^*)\big]=O(1/\sqrt{n}). 
\end{equation}
\end{theorem}

The first term on the right-hand side of Eq. \eqref{gen-rcd-l1-b} is related to estimation error, while the remaining two terms are related to optimization error.
According to \eqref{gen-rcd-l1-b}, we know that the estimation error bounds increase as we run more and more iterations, while the optimization errors would decrease. This suggests that we should balance these two errors by stoping the algorithm at an appropriate iteration to enjoy a favorable generalization, as shown in \eqref{gen-rcd-l1}.
\begin{remark}
  Under the same condition, it was shown that SGD with $T\asymp n$ can achieve the excess generalization bounds $O(1/\sqrt{n})$~\citep{hardt2016train}. Here we show that the same generalization bounds can be achieved by RCD. 
\end{remark}

In Theorem \ref{thm:gen-rcd-l1}, we require the boundedness assumption of stochastic gradients (note the bounded gradient assumption does not hold for the least square loss). We now show that this boundedness assumption can be removed by using the $\ell_2$-on-average argument stability.
A nice property is that it incorporates the information of $F(\bw^*)$ in the generalization bounds. This suggests that better generalization bounds can be achieved if $F(\bw^*)$ is small, which are called optimistic bounds in the literature~\citep{srebro2010smoothness,zhang2019stochastic}.
Here we introduce a parameter $\gamma$ to balance different components of the generalization bounds.
\begin{theorem}\label{thm:gen-rcd}
Let Assumptions \ref{ass:convex}, \ref{ass:smooth} hold.
Let $\{\bw_t\}$ be produced by \eqref{RCD} with nonincreasing $\eta_t\leq2/\widetilde{L}$. For any $\gamma>0$ such that $(1+T)(1+\gamma)L^2e\sum_{t=1}^{T}\eta_t^2\leq n^2d/4$, we have
\begin{multline}\label{gen-rcd}
  \ebb_{S,A}[F(\bw_T)-F_S(\bw^*)] =O\Big(\frac{1}{\gamma}+\frac{L^2(\gamma+\gamma^{-1}) T}{n^2d}\sum_{t=1}^{T}\eta_t^2\Big)\\
  \times F(\bw^*)
  + O\Big(\frac{d+d\gamma^{-1}}{\sum_{t=1}^{T}\eta_t}+\frac{L^2(\gamma+\gamma^{-1}) T}{n^2}\Big).
\end{multline}
\end{theorem}

The following corollary 
gives a quantitative suggestion on how to stop the algorithm for a good generalization. 
\begin{corollary}\label{cor:gen-rcd}
  Let Assumptions \ref{ass:convex}, \ref{ass:smooth} hold and $d=O(n^2)$. Let $\{\bw_t\}$ be produced by \eqref{RCD} with { $\eta_t\equiv\eta\leq2/\widetilde{L}$}.
  \begin{enumerate}[(a)]
    \item If $(1+T)(L+n\sqrt{d}/T)LeT\eta^2\leq n^2d/4$, then we can choose $T\asymp\sqrt{n}d^{\frac{3}{4}}$ to get\[\ebb_{S,A}[F(\bw_T)-F_S(\bw^*)]=O(d^{\frac{1}{4}}n^{-\frac{1}{2}}).\]
    \item If $F(\bw^*)=O(d^{\frac{1}{2}}Ln^{-1})$ and $(1+T)L^2eT\eta^2\leq n^2d/8$, we can choose $T\asymp n\sqrt{d}$ and get  \[\ebb_{S,A}[F(\bw_T)-F_S(\bw^*)]=O(d^{\frac{1}{2}}n^{-1}).\]
  \end{enumerate}
\end{corollary}
\begin{remark}
  If $T\!\asymp\!\sqrt{n}d^{\frac{3}{4}}$, then $(1\!+\!T)(L\!+\!n\sqrt{d}/T)LeT\eta^2\asymp nd^{\frac{3}{2}}+n^{\frac{3}{2}}d^{\frac{5}{4}}$. Then the assumption in Part (a) holds if $d\leq cn^2$ for some appropriate $c>0$.
  If $T\asymp n\sqrt{d}$, then $(1+T)L^2eT\asymp n^2d$. In this case, the assumption $(1+T)L^2eT\eta^2\leq n^2d/8$ in Part (b) is also easy to satisfy.
\end{remark}
\begin{remark}
  As compared to Theorem \ref{thm:gen-rcd-l1}, Part (a) admits a worse dependency on the dimensionality, which is the cost we pay for removing the Lipschitz continuity assumption. Furthermore,
  Part (b) shows that RCD is able to achieve a generalization bound as fast as $O(\sqrt{d}/n)$ if the best model has a small population risk, while Theorem \ref{thm:gen-rcd-l1} fails to exploit this low-noise assumption and can only imply at most the generalization bound $O(1/\sqrt{n})$.
\end{remark}

\noindent\textbf{Strongly convex case}. Now, we present generalization bounds for RCD for strongly convex objective functions. 

\begin{theorem}\label{thm:gen-bound-rcd-sc}
Let Assumptions \ref{ass:lipschitz}, \ref{ass:convex}, \ref{ass:smooth}, \ref{ass:sc} hold.
Let $\{\bw_t\}$ be produced by \eqref{RCD} with {$\eta_t\equiv\eta\leq1/\widetilde{L}$}. Then
\[
\ebb_{S,A}\big[F(\bw_{T+1})-F(\bw^*)\big]\leq \frac{4G_1G_2}{n\sigma}+\big(1-\eta\sigma/d\big)^TF(\bw_{1}).
\]
In particular, we can set $T\asymp d\sigma^{-1}\log1/(n\sigma)$ to get the excess generalization bound
\[\ebb_{S,A}\big[F(\bw_{T+1})-F(\bw^*)\big]=O(1/(n\sigma)).\]
\end{theorem} 
\begin{remark}
  Stability bounds of the order $O(1/(n\sigma))$ were established for SGD under a strongly convex setting~\citep{hardt2016train}, which together with optimization error bounds of the order $O(1/(T\sigma))$~\citep{rakhlin2012making}, shows that SGD with $n$ iterations can achieve excess risk bounds $O(1/(n\sigma))$. Here we show that this optimal generalization bound can also be achieved for RCD with $d\sigma^{-1}\log1/(n\sigma)$ iterations.
\end{remark}

\noindent\textbf{High probability generalization bounds}. Finally, we present high-probability bounds
which are much more challenging than bounds in expectation and are important to understand the variation of the algorithm in repeated runs.
\begin{theorem}\label{thm:gen-hp}
  Let Assumptions \ref{ass:lipschitz}, \ref{ass:convex}, \ref{ass:smooth}, \ref{ass:lipschitz-inf} hold. Let $\{\bw_t\}$ be produced by \eqref{RCD} with {$\eta_t\equiv\eta\leq2/\widetilde{L}$} and $\delta\in(0,1)$. Assume $\|\bw_t\|_\infty\leq R$ and $|f(\bw_t;z)|\leq R$ for all $t$. If we choose $T\asymp n^{\frac{2}{3}}d^{\frac{1}{3}}\log^{-\frac{2}{3}}n\log^{-\frac{1}{3}}(1/\delta)$, then with probability at least $1-\delta$ there holds
  \[
  F(\bar{\bw}_T)-F(\bw^*)=O\Big(\big(d/n\big)^{\frac{1}{3}}\log^{\frac{1}{3}}n\log^{\frac{2}{3}}(1/\delta)\Big),
  \]
  where $\bar{\bw}_T=\frac{1}{T}\sum_{t=1}^{T}\bw_t$ is an average of iterates.
\end{theorem}


\section{Experiments\label{sec:exp}}
\begin{figure*}[htbp]
  \vspace*{-0.006\textheight}
  \centering
  \hspace*{-0.6cm}
  \subfigure[Ionosphere]{\includegraphics[width=0.32\textwidth,trim=0 2 4 3, clip]{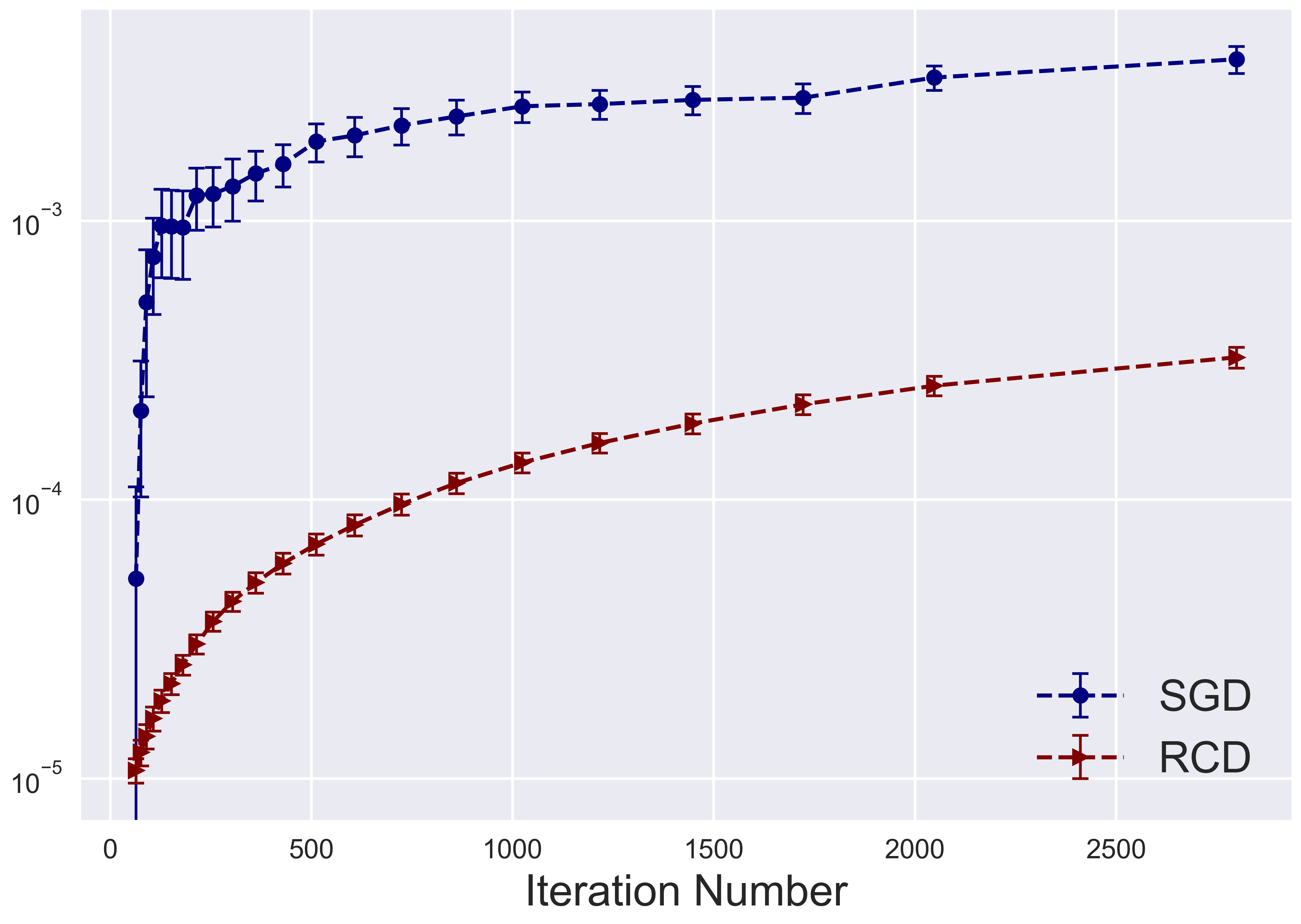}}\hspace*{0.2cm}
  \subfigure[Svmguide3]{\includegraphics[width=0.32\textwidth,trim=0 2 4 3, clip]{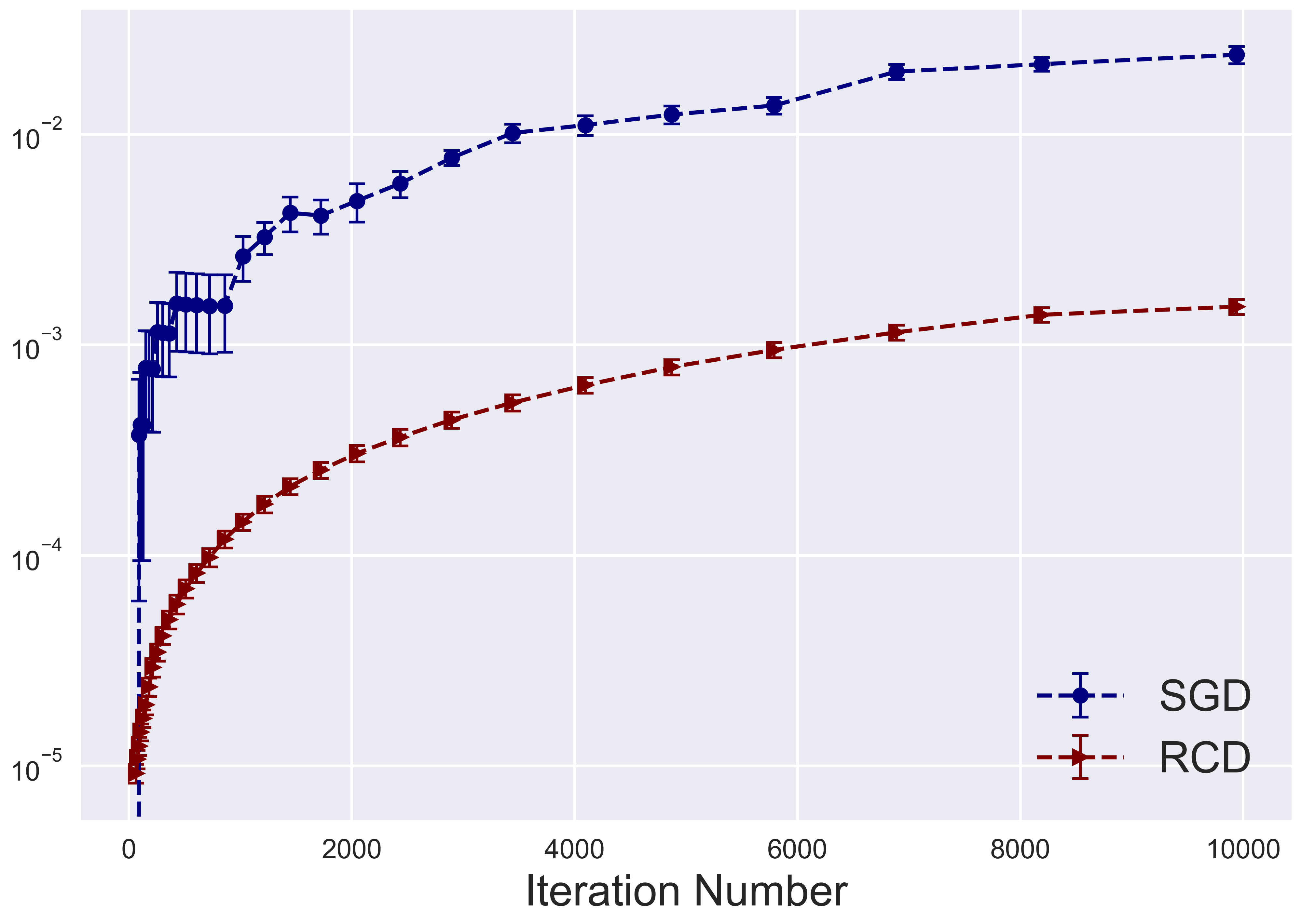}}\hspace*{0.2cm}
  \subfigure[MNIST]{\includegraphics[width=0.32\textwidth,trim=0 2 4 3, clip]{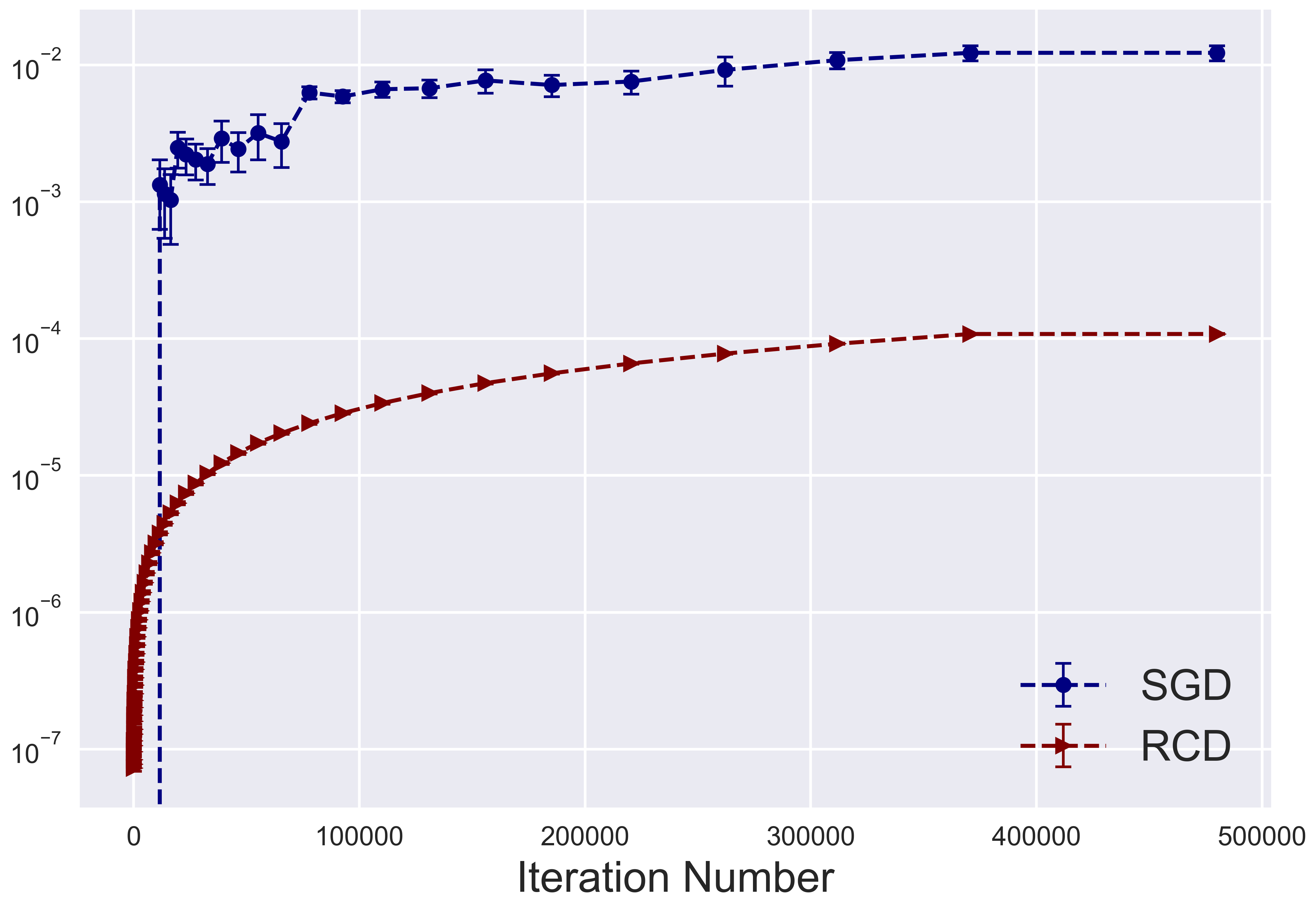}}
  \caption{Euclidean distance between two iterate sequences of RCD/SGD on neighboring datasets.\label{fig:stab}}
\end{figure*}
In this section, we present some experimental results to illustrate our stability bounds.
We follow the set up in \citet{hardt2016train}, i.e., we consider two neighboring datasets and run RCD/SGD with $\eta_t\equiv 0.01$ on these neighboring datasets to produce two iterate sequences $\{\bw_t\},\{\bw_t'\}$. We then plot the Euclidean distance between two iterate sequences as a function of the iteration number. We consider the least square regression for two datasets available at LIBSVM website~\citep{chang2011libsvm}: ionosphere, svmguide3 and MNIST. We repeat the experiments $100$ times and report the average of results. In Figure \ref{fig:stab} we plot the Euclidean distance as a function of the number of iterations.  Experimental results show that the Euclidean distance for RCD is much smaller than that with SGD, which is consistent with our theoretical results that RCD is more stable than SGD.

\section{Proof of Theorem \ref{thm:stab-bound-rcd}\label{sec:proof}}
The basic idea to prove Theorem \ref{thm:stab-bound-rcd} is to show how $\|\bw_{t+1}-\bw_{t+1}^{(i)}\|_2$ would change after a single iteration.
\begin{proof}[Proof of Theorem \ref{thm:stab-bound-rcd}]
We first prove Part (a). According to the update rule \eqref{RCD}, we know
\begin{align}
   & \|\bw_{t+1}-\bw_{t+1}^{(i)}\|_2 \notag\\
   & = \|\bw_t-\eta_t\nabla_{i_t}F_{S^{(i)}}(\bw_t)\be_{i_t}-\bw_t^{(i)}+\eta_t\nabla_{i_t}F_{S^{(i)}}(\bw_t^{(i)})\be_{i_t}\notag\\
   &+\eta_t\nabla_{i_t}F_{S^{(i)}}(\bw_t)\be_{i_t}-\eta_t\nabla_{i_t}F_S(\bw_t)\be_{i_t}\|_2 \notag\\
  & \leq \|\bw_t-\eta_t\nabla_{i_t}F_{S^{(i)}}(\bw_t)\be_{i_t}-\bw_t^{(i)}+\eta_t\nabla_{i_t}F_{S^{(i)}}(\bw_t^{(i)})\be_{i_t}\|_2\notag\\
  &+\eta_t\|\nabla_{i_t}F_{S^{(i)}}(\bw_t)\be_{i_t}-\nabla_{i_t}F_S(\bw_t)\be_{i_t}\|_2 \label{rcd-l1-01} \\
  & \leq
  \|\bw_t-\bw_t^{(i)}\|_2+\eta_t\|\nabla_{i_t}F_{S^{(i)}}(\bw_t)\be_{i_t}-\nabla_{i_t}F_S(\bw_t)\be_{i_t}\|_2,\label{rcd-l1-1}
\end{align}
where we have used Lemma \ref{lem:nonexpansive-rcd} in the last step.
Since $S$ and $S^{(i)}$ differ by the $i$-th example, we know
\begin{align}
  & |\nabla_{i_t}F_{S^{(i)}}(\bw_t)\!-\!\nabla_{i_t}F_S(\bw_t)|  = \frac{1}{n}\big|\nabla_{i_t}f(\bw_t;z_i)\!-\!\nabla_{i_t}f(\bw_t;z'_i)\big|\notag\\
  & \leq \frac{1}{n}\Big(\big|\nabla_{i_t}f(\bw_t;z_i)\big|+\big|\nabla_{i_t}f(\bw_t;z'_i)\big|\Big).\label{rcd-l1-7}
\end{align}
Note that $i_t$ is uniformly drawn from $[d]$, we further know
\begin{align}
  & \ebb_{i_t}\big[|\nabla_{i_t}F_{S^{(i)}}(\bw_t)-\nabla_{i_t}F_S(\bw_t)|\big] \notag\\
  & \leq \frac{1}{nd}\sum_{j=1}^{d}\Big(\big|\nabla_{j}f(\bw_t;z_i)\big|+\big|\nabla_{j}f(\bw_t;z'_i)\big|\Big)\notag \\
  & = \frac{1}{nd}\big(\|\nabla f(\bw_t;z_i)\|_1+\|\nabla f(\bw_t;z'_i)\|_1\big) \leq \frac{2G_1}{nd},\label{rcd-l1-2}
\end{align}
where we have used Assumption \ref{ass:lipschitz} in the last step. Plugging the above inequality back into \eqref{rcd-l1-1}, we get
\[
\ebb_A\big[\|\bw_{t+1}-\bw_{t+1}^{(i)}\|_2\big]
\leq \ebb_A\big[\|\bw_{t}-\bw_{t}^{(i)}\|_2\big]+\frac{2G_1\eta_t}{nd}.
\]
Applying the above inequality recursively gives the stated inequality. This completes the proof of Part (a).

We now prove Part (b).
According to \eqref{RCD}, we know
\begin{align}
   & \|\bw_{t+1}-\bw_{t+1}^{(i)}\|_2^2  \notag\\
   & = \|\bw_t-\eta_t\nabla_{i_t}F_S(\bw_t)\be_{i_t}-\bw_t^{(i)}+\eta_t\nabla_{i_t}F_{S^{(i)}}(\bw_t^{(i)})\be_{i_t}\|_2^2 \notag\\
   & = \|\bw_t-\eta_t\nabla_{i_t}F_{S^{(i)}}(\bw_t)\be_{i_t}-\bw_t^{(i)}+\eta_t\nabla_{i_t}F_{S^{(i)}}(\bw_t^{(i)})\be_{i_t}\notag\\
   &+\eta_t\nabla_{i_t}F_{S^{(i)}}(\bw_t)\be_{i_t}-\eta_t\nabla_{i_t}F_S(\bw_t)\be_{i_t}\|_2^2. \notag
\end{align}
By $(a+b)^2\leq (1+p)a^2+(1+1/p)b^2$ we know
\begin{align*}
&\|\bw_{t+1}-\bw_{t+1}^{(i)}\|_2^2\leq\\ &(1\!+\!p)\|\bw_t\!-\!\eta_t\nabla_{i_t}F_{S^{(i)}}(\bw_t)\be_{i_t}\!-\!\bw_t^{(i)}\!+\!\eta_t\nabla_{i_t}F_{S^{(i)}}(\bw_t^{(i)})\be_{i_t}\|_2^2\\
&+(1+1/p)\eta_t^2\|\nabla_{i_t}F_{S^{(i)}}(\bw_t)\be_{i_t}-\nabla_{i_t}F_S(\bw_t)\be_{i_t}\|_2^2.
\end{align*}
It then follows from Lemma \ref{lem:nonexpansive-rcd} that
\begin{multline}
  \|\bw_{t+1}-\bw_{t+1}^{(i)}\|_2^2 \leq (1+p)\|\bw_t-\bw_t^{(i)}\|_2^2\\
  +(1+1/p)\eta_t^2\|\nabla_{i_t}F_{S^{(i)}}(\bw_t)\be_{i_t}-\nabla_{i_t}F_S(\bw_t)\be_{i_t}\|_2^2,\label{rcd-1}
\end{multline}
Note that $S$ and $S^{(i)}$ differ by the $i$-th example, we can analyze analogously to \eqref{rcd-l1-7} and get
\begin{multline*}
  |\nabla_{i_t}F_{S^{(i)}}(\bw_t)-\nabla_{i_t}F_S(\bw_t)|^2 
  \\ \leq \frac{2}{n^2}\Big(\big|\nabla_{i_t}f(\bw_t;z_i)\big|^2+\big|\nabla_{i_t}f(\bw_t;z'_i)\big|^2\Big).
\end{multline*}
Since $i_t$ is uniformly drawn from $[d]$, we further know
\begin{align}
  & \ebb_{i_t}\big[|\nabla_{i_t}F_{S^{(i)}}(\bw_t)-\nabla_{i_t}F_S(\bw_t)|^2\big] \notag\\
  & \leq \frac{2}{n^2d}\sum_{j=1}^{d}\Big(\big|\nabla_{j}f(\bw_t;z_i)\big|^2+\big|\nabla_{j}f(\bw_t;z'_i)\big|^2\Big)\notag \\
  & = \frac{2}{n^2d}\big(\|\nabla f(\bw_t;z_i)\|_2^2+\|\nabla f(\bw_t;z'_i)\|_2^2\big)\notag \\
  & \leq \frac{4L}{n^2d}\big(f(\bw_t;z_i)+f(\bw_t;z'_i)\big),\notag
\end{align}
where we have used the self-bounding property according to the $L$-smoothness of $f$ in the last step. 
Putting the above inequality back into \eqref{rcd-1} implies
\begin{multline}\label{rcd-6}
\ebb_A\big[\|\bw_{t+1}-\bw_{t+1}^{(i)}\|_2^2\big]\leq
 (1+p)\ebb_A\big[\|\bw_t-\bw_t^{(i)}\|_2^2\big]\\+\frac{4(1+1/p)L\eta_t^2}{n^2d}\ebb_A\big[f(\bw_t;z_i)+f(\bw_t;z'_i)\big].
\end{multline}
It then follows that
\begin{multline*}
\ebb_A\big[\|\bw_{t+1}-\bw_{t+1}^{(i)}\|_2^2\big]\leq \\
\frac{4L(1+1/p)}{n^2d}\sum_{j=1}^{t}\big(1+p\big)^{t-j}\eta_j^2\ebb_A\big[f(\bw_j;z_i)+f(\bw_j;z'_i)\big].
\end{multline*}
Taking an average over $i$, we derive
\begin{align*}
& \frac{1}{n}\sum_{i=1}^{n}\ebb_A\big[\|\bw_{t+1}-\bw_{t+1}^{(i)}\|_2^2\big] \\
& \leq
\frac{4L(1+1/p)}{n^3d}\sum_{i=1}^{n}\!\sum_{j=1}^{t}\big(1\!+\!p\big)^{t-j}\eta_j^2\ebb_A\big[f(\bw_j;z_i)\!+\!f(\bw_j;z'_i)\big]\\
& = \frac{4L(1+1/p)}{n^2d}\sum_{j=1}^{t}\big(1+p\big)^{t-j}\eta_j^2\ebb_A\big[F_S(\bw_j)+F_{S'}(\bw_j)\big].
\end{align*}
The proof is complete.
\end{proof}

\section{Conclusions\label{sec:conclusion}}

In this paper, we initialize the generalization analysis of RCD based on the algorithmic stability.
We establish upper bounds of argument stability and uniform stability for RCD, which further imply the optimal generalization bounds of the order $O(1/\sqrt{n})$ and $O(1/n)$ in the convex and strongly convex case, respectively. We also consider nonconvex case and develop high-probability bounds. Remarkably, our analysis can leverage the low-noise assumption to yield optimistic generalization bounds $O(1/n)$ in the convex case without a bounded gradient assumption.


There are several interesting future directions. First, it would be interesting to extend our analysis to other variants, such as distributed RCD and RCD for structure optimization.  Second, here we assume the objectives are convex/strongly convex and each coordinate is sampled with the same probability during RCD updates. It is interesting to extend our discussion to nonconvex setting and importance sampling \citep{nesterov2012efficiency}, which are popular in modern machine learning.

\section*{Acknowledgments}
This work was supported in part by the National Natural Science Foundation of China (Grant Nos. 61903309, 61806091, 11771012, U1811461) and the Fundamental Research Funds for the Central Universities (JBK1806002).

\bibliographystyle{named}
\setlength{\bibsep}{0.111cm}

\newpage
\appendix
\numberwithin{equation}{section}
\numberwithin{theorem}{section}
\numberwithin{figure}{section}
\numberwithin{table}{section}
\renewcommand{\thesection}{{\Alph{section}}}
\renewcommand{\thesubsection}{\Alph{section}.\arabic{subsection}}
\renewcommand{\thesubsubsection}{\Roman{section}.\arabic{subsection}.\arabic{subsubsection}}
\setcounter{secnumdepth}{-1}
\setcounter{secnumdepth}{3}


\section{Some Lemmas}

We introduce some useful lemmas. The following lemma shows that the coordinate descent operator $\bw\mapsto\bw-\eta\nabla_ig(\bw)\be_i$ is non-expansive.
\begin{lemma}\label{lem:nonexpansive-rcd}
  Let $g:\rbb^d\mapsto\rbb$ be convex and have coordinate-wise Lipschitz continuous gradients with parameter $\widetilde{L}>0$. Then for any $\eta\leq2/\widetilde{L}$ and any $i\in[d]$ we have the following inequality for any $\bw$ and $\tilde{\bw}$
  \begin{equation}\label{nonexpansive-rcd-a}
  \|\bw-\eta\nabla_ig(\bw)\be_i-\tilde{\bw}+\eta\nabla_ig(\tilde{\bw})\be_i\|_2\leq\|\bw-\tilde{\bw}\|_2.
  \end{equation}
  Furthermore, if $g$ is $\sigma$-coordinate-wise strongly convex and $\eta\leq1/\widetilde{L}$, then
  \begin{multline}\label{nonexpansive-rcd-b}
  \|\bw-\eta\nabla_ig(\bw)\be_i-\tilde{\bw}+\eta\nabla_ig(\tilde{\bw})\be_i\|_2^2\\
  \leq \|\bw-\tilde{\bw}\|_2^2-\eta\sigma|w_i-\tilde{w}_i|^2,
  \end{multline}
  where $w_i$ denotes the $i$-th coordinate of $\bw\in\rbb^d$.
\end{lemma}
To prove Lemma \ref{lem:nonexpansive-rcd}, we introduce the following lemma due to \citet{hardt2016train}.
\begin{lemma}[\citealt{hardt2016train}\label{lem:nonexpansive}]
Assume the function $g:\rbb\mapsto\rbb$ is convex and $L$-smooth. Then for all $w$, $w'\in\rbb$ and $\eta\leq2/L$ we know
  \begin{equation}\label{non-expansive-rcd-a}
  |w-\eta \nabla g(w)-w'+\eta \nabla g(w')|\leq |w-w'|.
  \end{equation}
  Furthermore, if $g$ is $\sigma$-strongly convex and $\eta\leq 1/L$ there holds
  \begin{equation}\label{non-expansive-rcd-b}
  |w-\eta \nabla g(w)-w'+\eta \nabla g(w')|^2\leq (1-\eta\sigma)|w-w'|^2.
  \end{equation}
\end{lemma}

\begin{proof}[Proof of Lemma \ref{lem:nonexpansive-rcd}]
  Recall that $w_j$ denotes the $j$-th coordinate of $\bw\in\rbb^d$.
  It is clear that
  \begin{multline*}
   \|\bw-\eta\nabla_ig(\bw)\be_i-\tilde{\bw}+\eta\nabla_ig(\tilde{\bw})\be_i\|_2^2
     \\= \sum_{j\neq i}|w_j-\tilde{w}_j|^2+|w_i-\eta\nabla_ig(\bw)-\tilde{w}_i+\eta\nabla_ig(\tilde{\bw})|^2.
  \end{multline*}
  We can fix $w_j$ for all $j\neq i$ and then $g$ can be considered as a univariate function of $w_i$.
 Applying  \eqref{non-expansive-rcd-a} in Lemma \ref{lem:nonexpansive} to this univariate function, we have
  \[
  |w_i-\eta\nabla_ig(\bw)-\tilde{w}_i+\eta\nabla_ig(\tilde{\bw})|\leq |w_i-\tilde{w}_i|.
  \]
Combining the above two inequalities together, we  derive the stated inequality \eqref{nonexpansive-rcd-a}.

We now turn to \eqref{nonexpansive-rcd-b} in a strongly convex setting. Similarly,  \eqref{non-expansive-rcd-b} in Lemma \ref{lem:nonexpansive} implies that
  \begin{multline*}
  \|\bw-\eta\nabla_ig(\bw)\be_i-\tilde{\bw}+\eta\nabla_ig(\tilde{\bw})\be_i\|_2^2\leq \\
  \sum_{j\neq i}|w_j-\tilde{w}_j|^2+(1-\eta\sigma)|w_i-\tilde{w}_i|^2.
  \end{multline*}
This completes the proof.
\end{proof}
Smooth functions enjoy the self-bounding property, which means that the gradients can be bounded by the function values~\citep{srebro2010smoothness}.
\begin{lemma}\label{lem:self-bounding}
  If $g:\wcal\mapsto\rbb_+$ is $L$-smooth, then for all $\bw\in\wcal$ there holds
  $
  \|\nabla g(\bw)\|_2^2\leq 2Lg(\bw).
  $
\end{lemma}

\section{Proof of Theorem \ref{thm:stab-bound-rcd-sc}}

\begin{proof}[Proof of Theorem \ref{thm:stab-bound-rcd-sc}]
By \eqref{nonexpansive-rcd-b}, we know ($w_{t, i_t}$ is the $i_t$-th component of $\bw_t$ and $w_{t,i_t}^{(i)}$ is the $i_t$-th component of $\bw_t^{(i)}$)
\begin{align*}
  & \ebb_{i_t}\big[\|\bw_t-\eta_t\nabla_{i_t}F_{S^{(i)}}(\bw_t)\be_{i_t}-\bw_t^{(i)}+\eta_t\nabla_{i_t}F_{S^{(i)}}(\bw_t^{(i)})\be_{i_t}\|_2^2\big]\\
  & \leq \|\bw_t-\bw_t^{(i)}\|_2^2
  -\eta_t\sigma\ebb_{i_t}[|w_{t, i_t}-{w}_{t,i_t}^{(i)}|^2]\\
  & =(1-\eta_t\sigma/d)\|\bw_t-\bw_t^{(i)}\|_2^2
  \leq (1\!-\!\eta_t\sigma/(2d))^2\|\bw_t\!-\!\bw_t^{(i)}\|_2^2,
\end{align*}
where we have used
\[
\ebb_{i_t}[|w_{t, i_t}-w_{t,i_t}^{(i)}|^2]=\frac{1}{d}\sum_{k=1}^{d}[|w_t^{(k)}-w_{t,k}^{(i)}|^2]=\frac{1}{d}\|\bw_t-\bw_t^{(i)}\|_2^2
\]
and $1-a\leq(1-a/2)^2$. It then follows that
\begin{multline*}
\ebb_{i_t}\big[\|\bw_t-\eta_t\nabla_{i_t}F_{S^{(i)}}(\bw_t)\be_{i_t}-\bw_t^{(i)}+\eta_t\nabla_{i_t}F_{S^{(i)}}(\bw_t^{(i)})\be_{i_t}\|_2\big]\\
\leq (1-\eta_t\sigma/(2d))\|\bw_t-\bw_t^{(i)}\|_2.
\end{multline*}
Putting the above inequality into \eqref{rcd-l1-01} and using \eqref{rcd-l1-2}, we get
\[
\ebb_A[ \|\bw_{t+1}-\bw_{t+1}^{(i)}\|_2]\leq (1-\eta_t\sigma/(2d))\ebb_A\big[\|\bw_t-\bw_t^{(i)}\|_2\big]+\frac{2G_1\eta_t}{nd}.
\]
Applying the above inequality recursively gives
\[
\ebb_A[ \|\bw_{t+1}-\bw_{t+1}^{(i)}\|_2]\leq \frac{2G_1}{nd}\sum_{j=1}^{t}\eta_j\prod_{k=j+1}^{t}(1-\eta_k\sigma/(2d)).
\]
Note that
\begin{align*}
  & \sum_{j=1}^{t}\eta_j\sigma/(2d)\prod_{k=j+1}^{t}(1-\eta_k\sigma/(2d))\\
  & = \sum_{j=1}^{t}\big(1-\big(1-\eta_j\sigma/(2d)\big)\big)\prod_{k=j+1}^{t}(1-\eta_k\sigma/(2d))\\
  & = \sum_{j=1}^{t}\prod_{k=j+1}^{t}(1-\eta_k\sigma/(2d))-\sum_{j=1}^{t}\prod_{k=j}^{t}(1-\eta_k\sigma/(2d))\\
  & = 1-\prod_{k=1}^{t}(1-\eta_k\sigma/(2d))\leq1.
\end{align*}
We can combine the above two inequalities together and get the stated bound.
The proof is complete.
\end{proof}

\section{Proofs of Generalization Bounds\label{sec:proof-generalization}}
In this section, we present the proofs of generalization bounds for RCD.
\subsection{Convex Case}
We first consider generalization bounds in the convex case.
\begin{proof}[Proof of Theorem \ref{thm:gen-rcd-l1}]
  Let $A(S)=\bw_T$. It follows from Part (a) of {Lemma \ref{thm:gen-model-stab} }and Eq. \eqref{stab-bound-rcd-l1} that
  \[
  \ebb_{{S,A}}\big[F(\bw_T)-F_S(\bw_T)\big]\leq \frac{2G_1G_2}{nd}\sum_{t=1}^{T}\eta_t.
  \]
  This together with the optimization error bounds in \eqref{rcd-a} with $\bw=\bw^*$ and the error decomposition \eqref{decomposition} gives
  \begin{multline*}
    \ebb_{{S,A}}\big[F(\bw_T)-F(\bw^*)\big]\leq \frac{2G_1G_2}{nd}\sum_{t=1}^{T}\eta_t\\
    +\frac{d\big(\|\bw_1-\bw^*\|_2^2  + 2\eta_1\ebb[F_S(\bw_1)]\big)}{2\sum_{t=1}^{T}\eta_t}.
  \end{multline*}
  Since {$\eta_t\equiv\eta$} and $\ebb_{{S,A}}[F_S(\bw_1)]=F(\bw_1)$, we further get the stated bound \eqref{gen-rcd-l1-b}.
  The second bound then follows from the choice of $T$. The proof is complete.
\end{proof}
\begin{proof}[Proof of Theorem \ref{thm:gen-rcd}]
Taking expectation on both sides of \eqref{stab-bound-rcd} and noticing $\ebb_{S'}[F_{S'}(\bw_j)]=F(\bw_j)$ since $\bw_j$ is independent of $S'$, we derive
\begin{multline*}
\ebb_{S,S',A}\Big[\frac{1}{n}\sum_{i=1}^{n}\|\bw_{t+1}-\bw_{t+1}^{(i)}\|_2^2\Big]\\
    \leq  \frac{4L(1+1/p)}{n^2d}\sum_{j=1}^{t}\big(1+p\big)^{t-j}\eta_j^2\ebb_{S,A}\big[F_S(\bw_j)+F(\bw_j)\big].
\end{multline*}
We plug the above inequality into Part (b) of {Lemma \ref{thm:gen-model-stab} } with $A(S)=\bw_{t+1}$, and derive
\begin{multline*}
  \ebb_{S,A}\big[F(\bw_{t+1})-(1+\gamma^{-1})F_S(\bw_{t+1})\big]\leq\\
  \frac{2(1+1/p)L^2(1+\gamma)}{n^2d}
  \sum_{j=1}^{t}\big(1+p\big)^{t-j}\eta_j^2\ebb_{S,A}
  \big[F(\bw_j)\\-(1+\gamma^{-1})F_S(\bw_j)+(2+\gamma^{-1})F_S(\bw_j)\big].
\end{multline*}
  Let \[\delta_j=\max\big\{\ebb_{S,A}[F(\bw_j)-(1+\gamma^{-1})F_S(\bw_j)],0\big\}, \quad\forall j\in\nbb.\]
  Then, it follows that
\begin{align*}
    & \delta_{t+1} -\frac{2(1+1/p)L^2(1+\gamma)(2+\gamma^{-1})(1+p)^{t-1}}{n^2d}\sum_{j=1}^{t}\eta_j^2\\
    & \times \ebb_{S,A}[F_S(\bw_j)]\leq \frac{2(1+1/p)L^2(1+\gamma)(1+p)^{t-1}}{n^2d}\sum_{j=1}^{t}\eta_j^2\delta_j\\
    &\leq \frac{2(1+1/p)L^2(1+\gamma)(1+p)^{t-1}}{n^2d}\sum_{j=1}^{t}\eta_j^2\max_{1\leq\tilde{j}\leq t+1}\delta_{\tilde{j}}.
\end{align*}
Since the above inequality holds for all $t$ and the above upper bound of $\{\delta_t\}$ is an increasing function of $t$, we can set $p=1/t$ and derive
\begin{multline*}
  \max_{1\leq\tilde{j}\leq t+1}\delta_{\tilde{j}}\leq\frac{2(1+t)L^2(1+\gamma)e}{n^2d}\sum_{j=1}^{t}\eta_j^2\max_{1\leq\tilde{j}\leq t+1}\delta_{\tilde{j}}\\
  +
  \frac{2(1+t)L^2(1+\gamma)(2+\gamma^{-1})e\sum_{j=1}^{t}\eta_j^2\ebb_{S,A}[F_S(\bw_j)]}{n^2d},
\end{multline*}
where we have used $(1+1/t)^t\leq e$.
According to the assumption $(1+t)L^2(1+\gamma)e\sum_{j=1}^{t}\eta_j^2\leq n^2d/4$, we further get the following inequality for all $t=1,\ldots,T$
\begin{multline*}
  \max_{1\leq\tilde{j}\leq t+1}\delta_{\tilde{j}}\leq\frac{1}{2}\max_{1\leq\tilde{j}\leq t+1}\delta_{\tilde{j}}\\
  +
  \frac{2(1+t)L^2(1+\gamma)(2+\gamma^{-1})e\sum_{j=1}^{t}\eta_j^2\ebb_{S,A}[F_S(\bw_j)]}{n^2d},
\end{multline*}
and therefore
\begin{multline*}
\ebb_{S,A}[F(\bw_t)-F_S(\bw_t)]\leq \\
\gamma^{-1}\big(\ebb_{S,A}[F_S(\bw_t)]-\ebb_{S,A}[F_S(\bw)]+\ebb_{S,A}[F_S(\bw)]\big) + \\
\frac{4(1+t)L^2(1+\gamma)(2+\gamma^{-1})e}{n^2d}\sum_{j=1}^{t}\eta_j^2\ebb_{S,A}[F_S(\bw_j)].
\end{multline*}
Summing both sides by ${\ebb_{S,A}}[F_S(\bw_t)-F_S(\bw)]$ and using the decomposition $F_S(\bw_j)=F_S(\bw_j)-F_S(\bw)+F_S(\bw)$,
it then follows from \eqref{rcd-a} and \eqref{rcd-b} that
\begin{multline*}
  \ebb_{S,A}[F(\bw_t)-F_S(\bw)] \leq \gamma^{-1}\ebb_{S,A}[F_S(\bw)]\\
  + \frac{d(1+\gamma^{-1})}{2\sum_{j=1}^{t}\eta_j}\Big(\|\bw_1-\bw\|_2^2  + 2\eta_1F(\bw_1)\Big)+\\
  \frac{2(1+t)L^2(1+\gamma)(2+\gamma^{-1})e}{n^2}\Big(\eta_1\|\bw_1-\bw\|_2^2 + 2\eta_1^2F(\bw_1)\Big)\\
  +\frac{4(1+t)L^2(1+\gamma)(2+\gamma^{-1})e}{n^2d}\sum_{j=1}^{t}\eta_j^2\ebb_{S,A}[F_S(\bw)].
\end{multline*}
We can set $\bw=\bw^*$ in the above inequality and get the stated bound. The proof is complete.

\end{proof}

\begin{proof}[Proof of Corollary \ref{cor:gen-rcd}]
  We first prove Part (a).
  For the constant step size {$\eta_t\equiv\eta$}, the generalization bound in \eqref{gen-rcd} becomes
  \begin{multline}\label{rcd-5}
    \ebb_{S,A}[F(\bw_T)-F_S(\bw^*)]=O\Big(\frac{1}{\gamma}+\frac{L^2(\gamma+\gamma^{-1}) T^2}{n^2d}\Big)F(\bw^*)\\
    +O\Big(\frac{d+d\gamma^{-1}}{T}+\frac{L^2(\gamma+\gamma^{-1}) T}{n^2}\Big).
  \end{multline}
  We choose $\gamma=\frac{n\sqrt{d}}{TL}$. If \[(1+T)(L+n\sqrt{d}/T)LeT\eta^2\leq n^2d/4,\]
  then \eqref{rcd-5} holds and becomes
  \[\ebb_{S,A}[F(\bw_T)-F_S(\bw^*)]=O\Big(\frac{LT}{n\sqrt{d}} + { \big(\frac{LT}{n\sqrt{d}}\big)^3} +\frac{d}{T}+{\frac{L\sqrt{d}}{n}}\Big).\] We can further choose $T\asymp\sqrt{n}d^{\frac{3}{4}}$ and use $d=O(n^2)$ to derive \[\ebb_{S,A}[F(\bw_T)-F_S(\bw^*)]=O(d^{\frac{1}{4}}n^{-\frac{1}{2}}).\]

  We now turn to Part (b).
  If $F(\bw^*)=O(d^{\frac{1}{2}}Ln^{-1})$, we choose $\gamma=1$. If $(1+T)L^2eT\eta^2\leq n^2d/8$,
  \eqref{rcd-5} holds and becomes \[\ebb_{S,A}[F(\bw_T)-F_S(\bw^*)]=O\Big(\frac{L^3T^2}{n^3d^{1/2}}+\frac{d}{T}+\frac{L^2T}{n^2}+{\frac{L\sqrt{d}}{n}}\Big).\] We can further choose $T\asymp n\sqrt{d}/L$ and derive \[\ebb_{S,A}[F(\bw_T)-F_S(\bw^*)]=O(d^{\frac{1}{2}}Ln^{-1}).\]
\end{proof}

\subsection{Strongly Convex Case}
We now consider the generalization bounds for strongly convex objective functions. 
\begin{proof}[Proof of Theorem \ref{thm:gen-bound-rcd-sc}]
Let $A(S)=\bw_{T}$.  It follows from Part (a) of Theorem \ref{thm:gen-model-stab} and Eq. \eqref{stab-bound-rcd-sc} that
  \[
  \ebb_{{S,A}}\big[F(\bw_{T})-F_S(\bw_{T})\big]\leq \frac{4G_1G_2}{n\sigma}.
  \]
  This together with the optimization error bounds in \eqref{rcd-c} and the error decomposition \eqref{decomposition} gives
  \begin{multline*}
    \ebb_{{S,A}}\big[F(\bw_{T})-F(\bw^*)\big]\leq
    \frac{4G_1G_2}{n\sigma}\\+\big(1-\eta\sigma/d\big)^{T-1}\ebb\big[F_S(\bw_{1})-F_S(\bw_S)\big].
  \end{multline*}
  To get the excess generalization bound $O(1/(n\sigma))$, it suffices that
  \begin{align*}
  \big(1-\eta\sigma/d\big)^{T-1}F(\bw_{1})&\leq \exp(-(T-1)\eta\sigma/d\big)F(\bw_{1})\\
  &=O\big(1/(n\sigma)\big).
  \end{align*}
  The stated requirement on $T$ follows directly. The proof is complete. 
\end{proof}

\section{Proof of Theorem \ref{thm:stab-nonconvex}\label{sec:nonconvex}}
In this section, we prove stability bounds of RCD in a nonconvex case.
\begin{proof}[Proof of Theorem \ref{thm:stab-nonconvex}]
 Since $F_S$ has coordinate Lipschitz gradients, we know ($w_{t, i_t}$ denotes the $i_t$-th component of $\bw_t$ and $w_{t,i_t}^{(i)}$ denotes the $i_t$-th component of $\bw_t^{(i)}$)
\[
\big|\nabla_{i_t}F_{S^{(i)}}(\bw_t)-\nabla_{i_t}F_{S^{(i)}}(\bw_t^{(i)})\big|\leq \widetilde{L}\big|w_{t, i_t}-w_{t,i_t}^{(i)}\big|
\]
and therefore
\begin{align*}
&\|\bw_t-\eta_t\nabla_{i_t}F_{S^{(i)}}(\bw_t)\be_{i_t}-\bw_t^{(i)}+\eta_t\nabla_{i_t}F_{S^{(i)}}(\bw_t^{(i)})\be_{i_t}\|_2\\
&\leq \|\bw_t-\bw_t^{(i)}\|_2+\eta_t\|\nabla_{i_t}F_{S^{(i)}}(\bw_t)\be_{i_t}-\nabla_{i_t}F_{S^{(i)}}(\bw_t^{(i)})\be_{i_t}\|_2\\
&\leq \|\bw_t-\bw_t^{(i)}\|_2+\widetilde{L}\eta_t\big|w_{t, i_t}-w_{t,i_t}^{(i)}\big|.
\end{align*}
It then follows the uniform distribution of $i_t$ that
\begin{align*}
&\ebb_{i_t}\big[\|\bw_t-\eta_t\nabla_{i_t}F_{S^{(i)}}(\bw_t)\be_{i_t}-\bw_t^{(i)}+\eta_t\nabla_{i_t}F_{S^{(i)}}(\bw_t^{(i)})\be_{i_t}\|_2\big]\\
&\leq \|\bw_t-\bw_t^{(i)}\|_2+\widetilde{L}\eta_t\ebb_{i_t}\big[\big|w_{t, i_t}-w_{t,i_t}^{(i)}\big|\big]\\
&= \|\bw_t-\bw_t^{(i)}\|_2+\widetilde{L}\eta_td^{-1}\sum_{j=1}^{d}\big[\big|w_{t,j}-w_{t,j}^{(i)}\big|\big]\\
& = \|\bw_t-\bw_t^{(i)}\|_2+\widetilde{L}\eta_td^{-1}\|\bw_t-\bw_t^{(i)}\|_1\\
& \leq \big(1+\widetilde{L}\eta_td^{-\frac{1}{2}}\big)\|\bw_t-\bw_t^{(i)}\|_2,
\end{align*}
where we have used $\|\cdot\|_1\leq\sqrt{d}\|\cdot\|_2$. We can plug the above inequality and \eqref{rcd-l1-2} into \eqref{rcd-l1-01}, and get
\begin{multline*}
  \ebb_{{A}}\big[\|\bw_{t+1}-\bw_{t+1}^{(i)}\|_2\big]\leq \\
  \big(1+\widetilde{L}\eta_td^{-\frac{1}{2}}\big)\ebb_A\big[\|\bw_t-\bw_t^{(i)}\|_2\big]+\frac{2G_1\eta_t}{nd}.
\end{multline*}
We can apply the above inequality recursively and get
\[
\ebb_{{A}}\big[\|\bw_{t+1}-\bw_{t+1}^{(i)}\|_2\big]\leq \frac{2G_1}{nd}\sum_{j=1}^{t}\eta_j\prod_{k=j+1}^{t}\big(1+\widetilde{L}\eta_kd^{-\frac{1}{2}}\big).
\]
The proof is complete.
\end{proof}

\section{Proofs of Optimization Error Bounds\label{sec:proof-optimization}}
In this section, we prove optimization error bounds. The discussions follow \citep{nesterov2012efficiency} and we give the proof here for completeness. 
\begin{proof}[Proof of Lemma \ref{lem:opt-rcd}]
It follows from \eqref{RCD} that
\begin{align}
  F_S(\bw_{t+1}) & = F_S(\bw_t-\eta_t\nabla_{i_t}F_S(\bw_t)\be_{i_t}) \notag\\
   & \leq F_S(\bw_t) - \eta_t|\nabla_{i_t}F_S(\bw_t)|^2 + \frac{\widetilde{L}\eta_t^2}{2}|\nabla_{i_t}F_S(\bw_t)|^2 \notag\\
   & \leq F_S(\bw_t) - \eta_t|\nabla_{i_t}F_S(\bw_t)|^2/2,\label{opt-rcd-01}
\end{align}
where the first inequality holds since $F_S$ has coordinate-wise Lipschitz continuous gradients~\citep{nesterov2012efficiency}
and in the last inequality we have used $\eta_t\leq1/\widetilde{L}$. Taking an expectation w.r.t. $i_t$, we derive
\begin{align}
  \ebb_{i_t}\big[F_S(\bw_{t+1})\big]&\leq F_S(\bw_t) - \eta_t\sum_{j=1}^{d}|\nabla_{j}F_S(\bw_t)|^2/(2d)\notag\\
  &=F_S(\bw_t) - \eta_t\|\nabla F_S(\bw_t)\|_2^2/(2d).\label{rcd-4}
\end{align}

According to the update \eqref{RCD} again, we know
\begin{multline}
  \|\bw_{t+1}-\bw\|_2^2 = \|\bw_t-\eta_t\nabla_{i_t}F_S(\bw_t)\be_{i_t}-\bw\|_2^2\\
  =\|\bw_t-\bw\|_2^2 + \eta_t^2|\nabla_{i_t}F_S(\bw_t)|^2 +\\ 2\eta_t\langle\bw-\bw_t,\nabla_{i_t} F_S(\bw_t)\be_{i_t}\rangle.\label{opt-rcd-02}
\end{multline}
Taking an expectation w.r.t. $i_t$, we derive
\begin{align*}
  &\ebb_{i_t}[\|\bw_{t+1}-\bw\|_2^2]
   \leq \|\bw_t-\bw\|_2^2 \\
   &+ \frac{\eta_t^2}{d}\sum_{j=1}^{d}|\nabla_jF_S(\bw_t)|^2 + \frac{2\eta_t}{d}\sum_{j=1}^{d}\langle\bw-\bw_t,\nabla_j F_S(\bw_t)\be_j\rangle \\
   & = \|\bw_t-\bw\|_2^2 + \frac{\eta_t^2}{d}\|\nabla F_S(\bw_t)\|_2^2 + \frac{2\eta_t}{d}\langle\bw-\bw_t,\nabla F_S(\bw_t)\rangle\\
   & \leq \|\bw_t-\bw\|_2^2 + 2\eta_t\ebb_{i_t}\big[F_S(\bw_t)-F_S(\bw_{t+1})\big] \\
   &\qquad\qquad + \frac{2\eta_t}{d}\big(F_S(\bw)-F_S(\bw_t)\big),\\
\end{align*}
where in the last inequality we have used \eqref{rcd-4} and the convexity of $F_S$.
This together with the assumption $\eta_{t+1}\leq\eta_t$ gives
\begin{align}\label{rcd-2}
&2\eta_t\ebb_A[F_S(\bw_t)-F_S(\bw)] \notag\\
&\leq d\ebb_A\big[\|\bw_t-\bw\|_2^2\big] - d\ebb_A[\|\bw_{t+1}-\bw\|_2^2] \notag\\
&+2\eta_td\cdot \ebb_A[F_S(\bw_t)]-2\eta_{t+1}d\cdot\ebb_A[F_S(\bw_{t+1})].
\end{align}
Taking a summation of the above inequality gives
\begin{multline*}
2\sum_{j=1}^{t}\eta_j\ebb_A[F_S(\bw_j)-F_S(\bw)] \leq d\cdot \ebb_A\big[\|\bw_1-\bw\|_2^2\big]\\
  + 2\eta_1d\cdot \ebb_A\big[F_S(\bw_1)\big].
\end{multline*}
This together with $\ebb_A[F_S(\bw_{t+1})]\leq\ebb_A[F_S(\bw_t)]$ due to \eqref{rcd-4} implies
\begin{align*}
&\ebb_A[F_S(\bw_t)-F_S(\bw)]\\
& \leq \frac{1}{\sum_{j=1}^{t}\eta_j}\sum_{j=1}^{t}\eta_j\ebb_A[F_S(\bw_j)-F_S(\bw)]\\
& \leq \frac{d}{2\sum_{j=1}^{t}\eta_j}\Big(\|\bw_1-\bw\|_2^2  + 2\eta_1F_S(\bw_1)\Big).
\end{align*}
This proves \eqref{rcd-a}.
Similarly, we can multiply both sides of \eqref{rcd-2} by $\eta_t$ and use the assumption $\eta_{t+1}\leq\eta_t$ to derive
\begin{align*}
  & 2\eta_t^2\ebb_A[F_S(\bw_t)-F_S(\bw)] \\
  &\leq d\eta_t\ebb_A\big[\|\bw_t-\bw\|_2^2\big] - d\eta_{t+1}\ebb_A[\|\bw_{t+1}-\bw\|_2^2] \\
  & +
  2\eta_t^2d\ebb_A[F_S(\bw_t)]-2\eta_{t+1}^2d\ebb_A[F_S(\bw_{t+1})].
\end{align*}
Taking a summation of the above inequality gives
\begin{multline}\label{rcd-b}
2\sum_{j=1}^{t}\eta_j^2\ebb_A[F_S(\bw_j)-F_S(\bw)] \leq \\ d\eta_1\|\bw_1-\bw\|_2^2 + 2\eta_1^2dF_S(\bw_1).
\end{multline}

Now, we prove \eqref{rcd-c}. By the $\sigma$-strong convexity we know~\citep{nesterov2012efficiency}
\[
F_S(\bw_t)-F_S(\bw_S)\leq \frac{1}{2\sigma}\|\nabla F_S(\bw_t)-\nabla F_S(\bw_S)\|_2^2.
\]
Plugging the above inequality back into \eqref{rcd-4} and using $\nabla F_S(\bw_S)=0$, we get
\[
\ebb\big[F_S(\bw_{t+1})\big]\leq \ebb[F_S(\bw_t)]-\frac{\eta_t\sigma}{d}\ebb\big[F_S(\bw_t)-F_S(\bw_S)\big].
\]
Subtracting both sides by $F_S(\bw_S)$ gives the stated bound.
The proof is complete.
\end{proof}

\section{Proofs of Bounds with High Probability\label{sec:hp}}
We first prove Theorem \ref{thm:stab-bound-rcd-hp} on the uniform stability of RCD.
\begin{proof}[Proof of Theorem \ref{thm:stab-bound-rcd-hp}]
Let $S$ and $S^{(i)}$ be defined in Definition \ref{def:aver-stab}.
Let $\{\bw_t\}, \{\bw_t^{(i)}\}$ be produced by \eqref{RCD} with $\eta_t\leq2/\widetilde{L}$ based on $S$ and $S^{(i)}$, respectively.
According to \eqref{rcd-l1-1}, \eqref{rcd-l1-7}, we know
\begin{align*}
   &\|\bw_{t+1}-\bw_{t+1}^{(i)}\|_2 \\
   &\leq
  \|\bw_t-\bw_t^{(i)}\|_2+\frac{\eta_t}{n}\big(\big|\nabla_{i_t}f(\bw_t;z_i)\big|+\big|\nabla_{i_t}f(\bw_t;z'_i)\big|\big)\\
  & \leq \|\bw_t-\bw_t^{(i)}\|_2+\frac{2\eta_t\widetilde{G}}{n}.
\end{align*}
It then follows that
\[
\|\bw_{t+1}-\bw_{t+1}^{(i)}\|_2\leq \sum_{k=1}^{t}\frac{2\eta_k\widetilde{G}}{n}.
\]
According to Assumption \ref{ass:lipschitz}, we further get
\[
f(\bw_{t+1};z)-f(\bw_{t+1}^{(i)};z)\leq \sum_{k=1}^{t}\frac{2G_2\eta_k\widetilde{G}}{n}.
\]
The proof is complete.
\end{proof}

The following lemma establishes high-probability generalization bounds for uniformly stable algorithms~\citep{bousquet2019sharper}.
\begin{lemma}\label{lem:stab-gen-hp}
Assume $|f(A(S);z)|\leq R$ for some $R>0$ and $S, z\in\zcal$. Let $\delta\in(0,1)$. If $A$ is $\epsilon$-uniformly-stable almost surely, then with probability at least $1-\delta$ there holds
  \begin{multline*}
  \big|F(A(S))-F_S(A(S))\big|=O\Big(\epsilon\log n\log(1/\delta)+\\
  Rn^{-\frac{1}2}\sqrt{\log(1/\delta)}\Big).
  \end{multline*}
\end{lemma}
We now turn to optimization error bounds with high probability. To this aim, we need to use concentration inequalities to control a martingale difference sequence.
\begin{lemma}\label{lem:martingale}
  Let $z_1,\ldots,z_n$ be a sequence of random variables such that $z_k$ may depend on the previous random variables $z_1,\ldots,z_{k-1}$ for all $k=1,\ldots,n$. Consider a sequence of functionals $\xi_k(z_1,\ldots,z_k),k=1,\ldots,n$.
  Let \[\sigma_n^2=\sum_{k=1}^{n}\ebb_{z_k}\big[\big(\xi_k-\ebb_{z_k}[\xi_k]\big)^2\big]\] be the conditional variance and $\delta\in(0,1)$.
  Assume that $\xi_k-\ebb_{z_k}[\xi_k]\leq b$ for each $k$ and $\rho\in(0,1]$. With probability at least $1-\delta$ we have
    \begin{equation}\label{bernstein}
      \sum_{k=1}^{n}\xi_k-\sum_{k=1}^{n}\ebb_{z_k}[\xi_k]\leq \frac{\rho\sigma_n^2}{b}+\frac{b\log\frac{1}{\delta}}{\rho}.
    \end{equation}
\end{lemma}
\begin{lemma}\label{lem:xi}
  Let Assumptions \ref{ass:lipschitz}, \ref{ass:lipschitz-inf} hold and $\|\bw_t\|_\infty\leq R$. Introduce the martingale difference sequence
\begin{equation}\label{xi-t}
  \xi_t=\langle\bw-\bw_t,\nabla_{i_t} F_S(\bw_t)\be_{i_t}\rangle-\ebb_{i_t}\big[\langle\bw-\bw_t,\nabla_{i_t} F_S(\bw_t)\be_{i_t}\rangle\big].
\end{equation}
For any $\delta\in(0,1)$, with probability at least $1-\delta$ there holds
\[
\sum_{t=1}^{T}\xi_t\leq 2(R+\|\bw\|_\infty)\Big(G_2\sqrt{d^{-1}T\log(1/\delta)}+\widetilde{G}\log(1/\delta)\Big).
\]
\end{lemma}
\begin{proof}
It is clear that $\ebb_{i_t}[\xi_t]=0$ and
\begin{multline*}
\big|\xi_t\big|\leq \|\bw-\bw_t\|_\infty\|\nabla_{i_t} F_S(\bw_t)\be_{i_t}-\ebb_{i_t}[\nabla_{i_t} F_S(\bw_t)\be_{i_t}]\|_1\\
\leq 2(R+\|\bw\|_\infty)\widetilde{G}.
\end{multline*}
Furthermore, there holds
\begin{align*}
  \ebb_{i_t}[\xi_t^2] & \leq \frac{1}{d}\sum_{i=1}^{d}\langle\bw-\bw_t,\nabla_{i} F_S(\bw_t)\be_{i}\rangle^2 \\
  & \leq \|\bw-\bw_t\|_\infty^2\frac{1}{d}\sum_{i=1}^{d}\|\nabla_{i} F_S(\bw_t)\be_{i}\|_1^2\\
  & \leq\frac{(R+\|\bw\|_\infty)^2\|\nabla F_S(\bw_t)\|_2^2}{d}\\
  & \leq
  \frac{(R+\|\bw\|_\infty)^2G_2^2}{d}.
\end{align*}
We can apply Lemma \ref{lem:martingale} to derive the following inequality with probability at least $1-\delta$
\[
\sum_{t=1}^{T}\xi_t\leq \frac{\rho T(R+\|\bw\|_\infty)G_2^2}{2\widetilde{G}d}+\frac{2(R+\|\bw\|_\infty)\widetilde{G}\log\frac{1}{\delta}}{\rho}.
\]
The stated inequality then follows with probability at least $1-\delta$ by taking $\rho=\min\big\{\frac{2\widetilde{G}\sqrt{d\log(1/\delta)}}{\sqrt{T}G_2},1\big\}$.
\end{proof}

\begin{lemma}\label{lem:opt-hp}
Let Assumptions \ref{ass:lipschitz}, \ref{ass:convex}, \ref{ass:smooth}, \ref{ass:lipschitz-inf} hold.
Let $\{\bw_t\}$ be produced by \eqref{RCD} with $\eta_t=\eta\leq2/\widetilde{L}$ and $\delta\in(0,1)$.
If $\|\bw_t\|_2\leq R$, then the following inequality holds with probability at least $1-\delta$
\begin{multline*}
F_S(\bar{\bw}_T)-F_S(\bw)\leq \frac{d}{T}\Big(F_S(\bw_1)+(2\eta)^{-1}\|\bw_1-\bw\|_2^2\Big)\\
+2(R+\|\bw\|_\infty)\Big(G_2\sqrt{dT^{-1}\log(1/\delta)}+\widetilde{G}dT^{-1}\log(1/\delta)\Big).
\end{multline*}
\end{lemma}
\begin{proof}[Proof of Lemma \ref{lem:opt-hp}]
As a combination of \eqref{opt-rcd-01} and \eqref{opt-rcd-02}, we derive
\begin{multline*}
\|\bw_{t+1}-\bw\|_2^2
  \leq \|\bw_t-\bw\|_2^2 + 2\eta_tF_S(\bw_t)-2\eta_tF_S(\bw_{t+1})\\
   + 2\eta_t\langle\bw-\bw_t,\nabla_{i_t} F_S(\bw_t)\be_{i_t}\rangle
\end{multline*}
Note
\begin{multline*}
\ebb_{i_t}\big[\langle\bw-\bw_t,\nabla_{i_t} F_S(\bw_t)\be_{i_t}\rangle\big]=\\
\frac{1}{d}\langle\bw-\bw_t,\nabla F_S(\bw_t)\rangle\leq \frac{F_S(\bw)-F_S(\bw_t)}{d}
\end{multline*}
and $\eta=\eta_t$ for all $t$.
Therefore, there holds
\begin{multline*}
\|\bw_{t+1}-\bw\|_2^2
  \leq \|\bw_t-\bw\|_2^2 + 2\eta F_S(\bw_t)-2\eta F_S(\bw_{t+1})\\ + 2\eta d^{-1}(F_S(\bw)-F_S(\bw_t))+2\eta\xi_t,
\end{multline*}
where we introduce $\{\xi_t\}$ in \eqref{xi-t}.
Taking a summation of the above inequality then gives
\begin{multline*}
\sum_{t=1}^{T}d^{-1}(F_S(\bw_t)-F_S(\bw))\leq  F_S(\bw_1)\\
+(2\eta)^{-1}\|\bw_1-\bw\|_2^2+\sum_{t=1}^{T}\xi_t.
\end{multline*}
It then follows from Lemma \ref{lem:xi} that the following inequality holds with probability at least $1-\delta$
\begin{multline*}
\sum_{t=1}^{T}d^{-1}(F_S(\bw_t)-F_S(\bw))\leq  F_S(\bw_1)+(2\eta)^{-1}\|\bw_1-\bw\|_2^2\\
+2(R+\|\bw\|_\infty)\Big(G_2\sqrt{d^{-1}T\log(1/\delta)}+\widetilde{G}\log(1/\delta)\Big).
\end{multline*}
The stated bound then follows from the convexity of $F_S$.
\end{proof}

Putting the above optimization bounds and uniform stability bounds together, we now present the proof of Theorem \ref{thm:gen-hp} on generalization bounds with high probability.
\begin{proof}[Proof of Theorem \ref{thm:gen-hp}]
  Let $A(S)=\bar{\bw}_T$. Then it follows from the convexity of norm and Theorem \ref{thm:stab-bound-rcd-hp} that $A$ is
  $\frac{2G_2\widetilde{G}}{n}\sum_{t=1}^{T}\eta_t$-uniformly stable. This together with Lemma \ref{lem:stab-gen-hp} gives the following inequality with probability at least $1-\delta/3$
  \begin{multline*}
  \big|F(\bar{\bw}_T)-F_S(\bar{\bw}_T)\big|\\
  =O\Big(\frac{T\log n\log(1/\delta)}{n}+Rn^{-\frac{1}2}\sqrt{\log(1/\delta)}\Big).
  \end{multline*}
  By Lemma \ref{lem:opt-hp} the following inequality holds with probability at least $1-\delta/3$
  \[
  F_S(\bar{\bw}_T)-F_S(\bw^*)=O\Big(\frac{d\log(1/\delta)}{T}+\sqrt{\frac{d\log(1/\delta)}{T}}\Big).
  \]
  Furthermore, the standard Hoeffding's inequality implies the following inequality with probability at least $1-\delta/3$
  \[
  \big|F(\bw^*)-F_S(\bw^*)\big|=O\big(\sqrt{n^{-1}\log(1/\delta)}\big).
  \]
  We can combine the above three inequalities together and get
  \begin{multline*}
  F(\bar{\bw}_T)-F(\bw^*)=O\Big(\frac{T\log n\log(1/\delta)}{n}\\
  +\sqrt{n^{-1}\log(1/\delta)}+\sqrt{\frac{d\log(1/\delta)}{T}}\Big).
  \end{multline*}
  We choose $T\asymp n^{\frac{2}{3}}d^{\frac{1}{3}}\log^{-\frac{2}{3}}n\log^{-\frac{1}{3}}(1/\delta)$ to get the stated bound with probability at least $1-\delta$. The proof is complete.
\end{proof}



%

\end{document}